\documentclass[letterpaper, 10 pt, conference]{ieeeconf}  %

\IEEEoverridecommandlockouts                              %

\usepackage{enumerate}

\newif\ifmessy
\messyfalse

\newif\ifICRA %
\ICRAtrue

\newif\ifappendix %
\appendixtrue

\title{\LARGE \bf
Leveraging Symmetry to Accelerate Learning of Trajectory  \\ Tracking Controllers for Free-Flying Robotic Systems
}

\author{Jake Welde${^{*}}$, Nishanth Rao${^{*}}$, Pratik Kunapuli${^{*}}$, Dinesh Jayaraman, and Vijay Kumar %
\thanks{
${^*}$ indicates equal contribution.
The authors are with the GRASP Laboratory at the University of Pennsylvania, Philadelphia,
PA 19104 USA.
 emails: \texttt{\{jwelde, nishrao, pratikk, dineshj, kumar\}@seas.upenn.edu}.
We gratefully acknowledge the support of 
ARL DCIST CRA W911NF-17-2-0181, %
NSF Grant CCR-2112665, %
NVIDIA,
and the NSF Graduate Research Fellowship Program.
Open-source code is available at 
\url{https://pratikkunapuli.github.io/EQTrackingControl/}.
}}%

\usepackage{times}

\usepackage{multicol}
\usepackage{multirow}
\usepackage{array}
\usepackage{tabularx}
\newcolumntype{L}{>{\raggedright\arraybackslash}X}
\newcolumntype{R}{>{\raggedleft\arraybackslash}X}
\newcolumntype{C}{>{\centering\arraybackslash}X}

\usepackage{amsmath, amssymb}
\usepackage{relsize}
\usepackage{dsfont}
\usepackage{graphicx}

\usepackage{subfigure}

\DeclareMathOperator{\id}{id}
\DeclareMathOperator{\pr}{pr}

\usepackage{amsthm}
\usepackage{thmtools}
\renewcommand\thmcontinues[1]{\textit{continued}}

\declaretheoremstyle[notefont=\normalfont\itshape,bodyfont=\normalfont]{normaltext}

\newtheorem{theorem}{Theorem}

\declaretheorem[name=Definition,style=normaltext]{definition}
\declaretheorem[name=Remark,style=normaltext]{remark}
\declaretheorem[name=Example,style=normaltext]{example}

\newtheorem{lemma}{Lemma}

\newcommand\transition[0]{\tau}
\newcommand\dynamics[0]{f}
\newcommand\reward[0]{R}

\newcommand{\manifold}[1]{\mathcal{#1}}
\newcommand{\identity}[0]{{1}}
\newcommand{\dt}[0]{\, \mathrm{dt}}
\newcommand{\norm}[1]{\|{#1}\|}

\usepackage{cite}

\usepackage{balance}

\usepackage{makecell}

\usepackage{tikz, tikzscale}
\usepackage{tikz-3dplot}
\usetikzlibrary{shapes,calc,positioning}

\usepackage[final,bookmarks=true, hidelinks]{hyperref}

\usepackage[addedmarkup=uline, defaultcolor=magenta, todonotes={textsize=scriptsize, textwidth=2cm}, authormarkuptext=name, commandnameprefix=always, xcolor]{changes} %
\definechangesauthor[name={JD}, color=orange]{jd}

\definechangesauthor[name={JW}, color=green]{jw}

\definechangesauthor[name={PK}, color=blue]{pk}

\definechangesauthor[name={NR}, color=gray]{nr}

\begin{document}

\maketitle

\begin{abstract}
    Tracking controllers enable robotic systems to accurately follow planned reference trajectories. In particular, reinforcement learning 
    (RL)
    has shown promise in the synthesis of controllers for systems with complex dynamics and modest online compute budgets.
    However, 
    the
    poor sample efficiency 
    of RL 
    and the challenges of reward design make training slow and sometimes unstable, especially for high-dimensional systems.
    In this work, we leverage the  inherent Lie group symmetries of robotic systems with a floating base
    to mitigate these challenges when learning tracking controllers.
    We model a general tracking problem as a Markov decision process (MDP) that captures the evolution of both the physical and reference states. Next, we 
    \ifICRA
    prove
    \fi
    that symmetry in the underlying dynamics and running costs leads to an MDP homomorphism,
    a mapping that allows a policy trained on a lower-dimensional ``quotient'' MDP to be lifted to an optimal tracking controller for the original system.
    We compare this symmetry-informed approach to an unstructured baseline, using Proximal Policy Optimization (PPO) to learn tracking controllers for three systems: the {\normalfont \texttt{Particle}} (a forced point mass), the {\normalfont \texttt{Astrobee}} (a fully-actuated space robot), and the {\normalfont \texttt{Quadrotor}} (an underactuated system). Results show that a symmetry-aware approach both accelerates training and 
    reduces tracking error 
    at convergence.

\end{abstract}

\section{Introduction}

To achieve real-time operation, most robotic systems utilize a {``tracking controller''} to stabilize a pre-planned {reference trajectory}.
However, tracking controllers designed analytically often assume properties not enjoyed by all robotic systems (\textit{e.g.},  ``full actuation'' \cite{Bullo1999,Maithripala2006,Welde2024} or ``differential flatness'' \cite{Fliess1997}), while optimization-based methods frequently rely on linearization or simplified models to meet compute constraints \cite{Nguyen2024}. In contrast, 
 controllers trained via reinforcement learning (RL) have relaxed structural assumptions while enabling real-time operation with moderate resources 
\cite{Hwangbo2017ControlOA}. 
In \cite{Molchanov2019}, the authors 
train
a single hovering policy for deployment across a range of quadrotors, generalizing satisfactorily to moving  references.
Meanwhile,
massively parallel training of quadrupedal walking policies from high-dimensional observations enabled startling robustness to uneven terrain \cite{pmlr-v164-rudin22a}, and learned controllers  augmented with adaptive feedforward compensation have been shown to reject large disturbances \cite{Huang23}. 
Unfortunately, these benefits come at a price:
RL tends to scale poorly with the size of the given Markov decision process (MDP), making it challenging to perform the exploration needed to discover high-performance policies.

To mitigate this burden, an RL agent should share experience across all those states that can be considered ``equivalent'' with respect to the reward and dynamics.
Indeed, robotic systems enjoy substantial symmetry \cite{Murray1997, Ostrowski1999, OrdonezApraez2023}, which has been thoroughly exploited in analytical control design \cite{Hatton2022, Welde2023, Hampsey2023quadrotor} and optimization \cite{Teng2022}.
In fact, many learned controllers have leveraged symmetry in an ad hoc or approximate manner (\textit{e.g.}, penalizing the \textit{error} between actual and reference states \cite{Molchanov2019} or working in the body frame \cite{Huang23}). 
More formally, the optimal policy of an MDP with symmetry is equivariant (and its value function is invariant) \cite{pmlr-v164-wang22j}, and neural architectures can be designed accordingly to improve sample efficiency and generalization \cite{vanderPol2020}.

Instead of incorporating symmetry into the network architecture, \cite{Ravindran2004} proposed ``MDP homomorphisms'', which establish a mapping from the given MDP to one of lower dimension. There, a policy may be trained more easily (using standard tools) and then lifted back to the original setting. Such methods were originally restricted to discrete state and action spaces, necessitating coarse discretization of robotic tasks (which are naturally described on smooth manifolds).
\cite{Yu2023} explored related ideas in continuous state and action spaces, but assumed deterministic dynamics (whereas stochasticity is fundamental to many tasks). However, \cite{Panangaden2024} recently extended the theory of homomorphisms of stochastic MDPs to the continuous setting,  recovering analogous value equivalence and policy lifting results. They also {learned} {approximate} homomorphisms from data, but do not give
a sufficient condition
to construct a well-behaved homomorphism (\textit{i.e.}, for which the new state and action spaces are also smooth manifolds) from a continuous symmetry known \textit{a priori} (as is the case for free-flying robotic systems \cite{Ostrowski1999}).

In this work, we explore the role of the {continuous} symmetries of free-flying robotic systems
in learned tracking control. After reviewing mathematical preliminaries in Sec. \ref{section:background}, in Sec. \ref{section:tracking_control_symmetry} we cast a general tracking control problem as a continuous MDP, using a stochastic process to model the (\textit{a priori} unknown) reference trajectory. We show that this MDP inherits the symmetry enjoyed by the underlying dynamics and running costs, and in Sec. \ref{section:mdp_homomorphisms} we
\ifICRA
prove
\fi
that such symmetries can be used to construct an MDP homomorphism, reducing the dimensionality. 
\ifICRA
In Sec. \ref{section:application}, we formally apply this method to three physical systems (including aerial and space robots).
\fi
Finally, in  Sec. \ref{section:experiments} we use these tools to learn tracking controllers for the example systems, accelerating training, improving tracking accuracy, and generalizing zero-shot to new trajectories. We discuss our results and contributions in Secs. \ref{section:discussion}-\ref{section:conclusion}. Ultimately, these insights will facilitate the efficient development of accurate tracking controllers for various robotic systems.

\section{Background and Preliminaries}
\label{section:background}

We now introduce some mathematical concepts.  
$\mathsf{B}({\manifold{X}})$ denotes the Borel {$\sigma$-algebra} of $\manifold{X}$, and 
$\Delta(\manifold{X})$ denotes the set of Borel probability measures on $\manifold{X}$
(see \cite[Appx. B]{Panangaden2024}).
Throughout the paper, we largely follow the treatment of \cite{Panangaden2024}, which (along with their prior work \cite{RezaeiShoshtari2022}) extends \cite{Ravindran2004} to study homomorphisms of Markov decision processes with continuous (\textit{i.e.}, not discrete) state and action spaces.

\begin{definition}[see \cite{Panangaden2024}]
	A \textit{continuous Markov decision process}\footnote{
 The more general definition in \cite{Panangaden2024} does \textit{not} assume $\manifold{S}$ and $\mathcal{A}$ are smooth manifolds, nor that $\tau( \,\cdot\,|\, s,a)$ is a Borel measure, but this is all we need.
 } (\textit{i.e.}, an MDP) is a tuple ${\manifold{M} = (\manifold{S},\manifold{A},\reward,\transition,\gamma)}$, where:
\begin{itemize}
\setlength{\itemindent}{-.5em}
	\item the \textit{state space} ${\manifold{S}}$ is a smooth manifold,
\item the \textit{action space} ${\manifold{A}}$ is a smooth manifold,
\item the \textit{instantaneous reward} is ${\reward : \manifold{S} \times \manifold{A} \to \mathbb{R}}$,
\item the \textit{transition dynamics} are ${{\transition} : \manifold{S} \times \manifold{A}  \to \Delta(\manifold{S})}$, and 
\item the \textit{discount factor} $\gamma$ is a value in the interval $[0,1)$.
	\end{itemize}	
After taking action ${a_t}$ from state ${s_t}$,
the probability that $s_{t+1}$ is contained in a set ${B \in \mathsf{B}(\manifold{S})}$ 
is given by $\tau( B \,|\, s_t,a_t)$.
A \textit{policy} for $\manifold{M}$ is a map 
${\pi : \manifold{S} \to \Delta(\manifold{A})}$.
The \textit{action-value function} $		Q^\pi : \manifold{S} \times \manifold{A} \to \mathbb{R}$ of a given policy $\pi$ is defined by
\begin{equation}
	\begin{gathered}
		Q^\pi(s,a) := 
		\mathop{
		\mathlarger{\mathbb{E}}
		}_{\transition \sim \pi}
		\left[\,
			\sum_{t=0}^\infty
			\gamma^t \reward(s_t,a_t)	
			\, \Big| \,
			s_0 = s, a_0 = a 
			\,
		\right],
	\end{gathered}
\end{equation}
where  ${\transition \sim \pi}$ denotes the expectation over both the transitions and the policy (\textit{i.e.},
${s_{t+1} \sim \transition(\, \cdot \,|\, s_t, a_t)}$ and ${a_t \sim \pi(\, \cdot \,|\, s_t)}$ for all $t \in \mathbb{N}$).
A policy $\pi^*$ is \textit{optimal} if, for all ${s \in \manifold{S}}$, 
\begin{equation}
	\pi^* = \arg \max_\pi 
		\mathop{\mathlarger{\mathbb{E}}}
		_{\transition \sim \pi}
		\left[
		\,
		\sum_{t=0}^\infty
		\gamma^t R(s_t,a_t)	
			\, \Big| \,
			s_0 = s 
		\,
		\right].
\end{equation}
\end{definition}

\subsection{Homomorphisms of Markov Decision Processes}

The following notion describes a powerful link between two continuous MDPs of (perhaps) different dimensions.

\begin{definition}[see {\cite[Defs. 11 and 14]{Panangaden2024}}]
A pair of maps
	 ${p : \manifold{S} \to \widetilde{\manifold{S}}}$ and
	 	 ${h : \manifold{S} \times \manifold{A} \to \widetilde{\manifold{A}}}$
	 	 is called a \textit{continuous MDP homomorphism}
	 	 from ${\manifold{M} = (\manifold{S},\manifold{A},\reward,\transition,\gamma)}$ to ${\widetilde{\manifold{M}} = (\widetilde{\manifold{S}},\widetilde{\manifold{A}},\widetilde{\reward},\widetilde{\transition},\gamma)}$
	 	 if
	 	 	$p$ and, for each ${s \in \manifold{S}}$,
the map	 ${h_s : a \mapsto h(s,a)}$ are measurable,  surjective maps, such that
 \begin{subequations}
 			\begin{align}
 		\reward\big(s,a\big) &= \widetilde{\reward}\big(p(s),h(s,a)\big),
 		\label{reward_invariance_mdp_homomorphism}
 		\\
 		 			{\transition} \big(p^{-1}(\widetilde{B}) \,|\, s,a\big)
&=
\widetilde{\transition}\big(\widetilde{B} \,|\, p(s),h(s,a)\big)
 		\label{transition_invariance_mdp_homomorphism}
 	\end{align}
 \end{subequations}	 
  for all ${s \in \manifold{S}}$, ${a \in \manifold{A}}$, and ${\widetilde{B} \in \mathsf{B}({\widetilde{S}}})$.
Given a continuous MDP homomorphism $(p,h)$, a policy  $\widetilde{\pi}$  for  $\widetilde{\manifold{M}}$, and a policy $\pi$ for  $\manifold{M}$, 
$\pi$ is called a \textit{lift of $\widetilde{\pi}$} if
for all 
${s \in \manifold{S}}$ and ${A \in \mathsf{B}(\manifold{A})}$,
\begin{align}
   \pi \big(h_s^{-1}(\widetilde{A}) \,|\, s\big) 
= \widetilde{\pi}\big(\widetilde{A} \,|\, p(s) \big).
    \end{align}
\end{definition}

Subsequently, we often omit the word ``continuous'' for brevity.
MDP homomorphisms 
facilitate the synthesis of an optimal policy for the original MDP $\mathcal{M}$ from an optimal policy for the ``quotient'' MDP $\widetilde{\mathcal{M}}$, via the following theorem.

\begin{theorem}[see {\cite[Thms. 12 and 16]{Panangaden2024}}]
    Suppose $(p,h)$ is an MDP homomorphism from $\manifold{M}$ to $\widetilde{\manifold{M}}$ and   
    $\pi$ is a lift of {\normalfont any} policy $\widetilde{\pi}$ for $\widetilde{\manifold{M}}$.
    Then,
${		Q^\pi(s,a) = \widetilde{Q}^{\widetilde{\pi}}\big(p(s),h(s,a)\big). }$
Moreover, if $\widetilde{\pi}$ is optimal for $\widetilde{\manifold{M}}$, then $\pi$ is optimal for $\manifold{M}$.
\label{mdp_homomorphism_value_equivalence_and_optimality}
\end{theorem}

\subsection{Lie Group Symmetries of Markov Decision Processes}

 A \textit{(left) group action} of a Lie group $\manifold{G}$ on a smooth manifold $\manifold{X}$ is a smooth map ${\Phi : \manifold{G} \times \manifold{X} \to \manifold{X}}$
(often written ${\Phi_g(x) := \Phi(g,x)}$ for brevity)
 such that for all ${x \in \manifold{X}}$ and ${g, h \in \manifold{G}}$, ${\Phi(\identity_{\manifold{G}},x) = x}$ (where ${\identity_{\manifold{G}}	 \in \manifold{G}}$ is the identity) and 
${\Phi\big(g,\Phi(h,x)\big) = \Phi\big(g h,x)}$.
The $\Phi$-\textit{orbit} of $x$ is the set ${\Phi_G(x) := \{\Phi_g(x) : g \in \manifold{G}\}}$, while  
$\manifold{X} / \manifold{G}$ is a set whose elements are all the orbits of $\Phi$.
An action $\Phi$ is \textit{proper} if the map $(g,x) \mapsto \big(\Phi_g(x),x\big)$ is proper (\textit{i.e.}, the preimage of any compact set is compact%
), 
and
\textit{free} if ${\Phi_g(x) = x}$ implies ${g = \identity_{\manifold{G}}}$.
A group $\manifold{G}$ acts on itself via ${L : (h,g) \mapsto hg}$.

A group action can describe a symmetry of some object defined on the manifold. We now formulate the following definition of a Lie group symmetry of a continuous MDP.

\begin{definition}
	 Given an MDP 
	 	${\manifold{M} = (\manifold{S},\manifold{A},\reward,\transition,\gamma)}$,
	 	a pair of Lie group actions 
	 	${(\Phi, \Psi)}$ of  $\manifold{G}$ on $\manifold{S}$ and $\manifold{A}$ respectively
	 	is called a \textit{Lie group symmetry} of $\manifold{M}$
	 if,
 	for all ${\Phi}$-invariant sets ${{B} \in \mathsf{B}(\manifold{S})}$ and all
	${s \in \manifold{S}}$, ${a \in \manifold{A}}$, and ${g \in \manifold{G}}$, we have
\begin{subequations}
	\begin{align}
		R\big(s,a\big) &= {R}\big(\Phi_g(s),\Psi_g(a)\big), \label{reward_invariance_lie_group_symmetry}
		\\
		{\transition}\big(B \,|\, s,a\big) &= {\transition}\big(B \,|\, \Phi_g(s),\Psi_g(a)\big). 
		\label{transition_invariance_lie_group_symmetry}
	\end{align}
\end{subequations}
\end{definition}

\begin{remark}
    The qualifier ``$\Phi$-invariant'' on $B$ broadens the class of symmetries considered (and is more general than \cite{pmlr-v164-wang22j} and \cite{vanderPol2020}, as noted in \cite[Def. 35]{Zhao2022thesis}). 
    The deterministic case (\textit{i.e.}, when $\transition(\, \cdot \, | \, s_t,a_t)$ is the Dirac measure corresponding to ${\{s_{t+1}\} \subseteq \manifold{S}}$) gives the intuition, since then  \eqref{transition_invariance_lie_group_symmetry} requires the image of any orbit in ${\mathcal{S} \times \mathcal{A}}$ to lie within some orbit in $\mathcal{S}$, \textit{without} enforcing equivariance \textit{within} each orbit. 
\end{remark}

\section{Tracking Control Problems \\ with Lie Group Symmetries}

\label{section:tracking_control_symmetry}

In this section, we formulate a general trajectory tracking problem as an MDP that models the evolution of both the physical and reference systems. We give a sufficient condition for this MDP to have a Lie group symmetry that  will be used (in Sec. \ref{section_MDPH}) to reduce the problem size.

\begin{definition}
	A \textit{tracking control problem} is a tuple  ${\manifold{T} = (\manifold{X}, \manifold{U}, \dynamics, J_{\manifold{X}}, J_{\manifold{U}}, \rho, \gamma)}$, where:
	\begin{itemize}
		\item $\manifold{X}$ is the \textit{physical state space} (a smooth manifold),
		\item $\manifold{U}$ is the \textit{physical action space} (a smooth manifold),
		\item ${\dynamics : \manifold{X} \times \manifold{U} \to \Delta(\manifold{X})}$ is the the \textit{physical dynamics} (\textit{i.e.}, ${x}_{t+1} \sim \dynamics( \, \cdot \, | \, x_t,u_t)$ describes the system's evolution),
		\item  ${J_{\manifold{X}} : \manifold{X} \times \manifold{X} \to \mathbb{R}}$ is the \textit{tracking cost}, \item ${J_{\manifold{U}} : \manifold{U} \times \manifold{U} \to \mathbb{R}}$ is the \textit{effort cost}, 
				\item  $\rho \in \Delta(\manifold{U})$ is the \textit{reference action distribution}, and
		\item ${\gamma \in [0,1)}$ is the \textit{discount factor}.
	\end{itemize}
\end{definition}

The distribution $\rho$ is not usually included in the definition of a tracking problem but will play an essential role in our approach (see Remark \ref{justifying_mdp_formulation}). 
\ifICRA
Going forward, we will use the following system as a running example to illustrate the theoretical concepts and the impact of a symmetry-informed approach (even for a simple system).

\begin{example}[name={\normalfont\texttt{Particle}\itshape},label=example:particle]
\label{particle_example}
	Consider a particle in $\mathbb{R}^3$  with mass $m$
 subject to a controlled external force (sometimes used as a reduced-order model for a quadrotor or rocket \cite{Huang23}). 
	The state ${x = (r,v) \in \manifold{X} = T\mathbb{R}^3 \simeq \mathbb{R}^3 \times \mathbb{R}^3}$ consists of the particle's position and velocity, and the control input is the applied force ${u \in \manifold{U} = \mathbb{R}^3}$. 
	The (deterministic) equations of motion, when discretized with timestep $\mathrm{dt}$,  are given by
	\begin{align}
		r_{t+1} = r_t + v_t \dt, \quad v_{t+1} = v_t + \tfrac{1}{m} u_t \dt,
	\end{align}
	so the transition probabilities ${f : T\mathbb{R}^3 \times \mathbb{R}^3 \to \Delta(T\mathbb{R}^3)}$ are
	\begin{equation}
		\hspace{-1pt}
		f( B \,|\, x,u) := 
		\begin{cases}
			1, \, (r + v \dt, v + \tfrac{1}{m} u \dt) \in B,
			\\ 0, \, \textrm{otherwise}.
		\end{cases}
		\hspace{-1pt}
	\end{equation}
	For some ${c_r,c_v,c_u \geq 0}$, we define the running costs
	\begin{subequations}
 \label{point_particle_running_costs}
		\begin{align}
      			J_{T\mathbb{R}^3}
   \big((r,v),(r^{\mathrm{d}},v^{\mathrm{d}})\big)
   &:= \alpha(r - r^{\mathrm{d}}) +  c_v \norm{v - v^{\mathrm{d}}},
\label{point_particle_tracking_cost}
			\\
			J_{\mathbb{R}^3}(u,u^{\mathrm{d}}) &:= c_u \norm{u - u^{\mathrm{d}}},
\label{point_particle_effort_cost}
		\end{align}
  \end{subequations}
  where 
$
{\alpha(y) := c_r\norm{y} +\tanh(a_r\norm{y}) - 1}$.
	Selecting a covariance $\Sigma$ and a discount factor ${0 \leq \gamma < 1}$, we define the tracking problem 
	${\manifold{T} = \big(T\mathbb{R}^3, \mathbb{R}^3, \dynamics, J_{T\mathbb{R}^3}, J_{\mathbb{R}^3}, \mathcal{N}(0,\Sigma), \gamma\big)}$.

\end{example}
\fi

\subsection{Modeling a Tracking Control Problem as an MDP}

We choose to model the trajectory tracking task for \textit{a priori} unknown reference trajectories 
 in the following manner. 

\begin{definition} \label{tracking_mdp_definition}
	A given tracking control problem ${\manifold{T} =}$ ${ \big(\manifold{X}, \manifold{U}, \dynamics, J_{\manifold{X}}, J_{\manifold{U}},  \rho, \gamma\big)}$ induces  a \textit{tracking control MDP}
	 given by ${\manifold{M}_{\manifold{T}} = ({\manifold{S} = \manifold{X} \times \manifold{X} \times \manifold{U}},\manifold{A} = \manifold{U},R,\transition,\gamma)}$, where:
\begin{itemize}
    \item the state is ${(x,x^{\mathrm{d}},u^{\mathrm{d}})}$, where 
    ${x}$, ${x^{\mathrm{d}} \in \manifold{X}}$ are the \textit{actual} and \textit{reference states} and ${ u^{\mathrm{d}} \in \manifold{U}}$ is the \textit{reference action},
    \item the {actions} are ${a = u \in \manifold{U}}$ (\textit{i.e.}, the \textit{actual action}),	 		  \item  the instantaneous reward 
	 		$						{{R} : \manifold{S} \times \manifold{A} \to \mathbb{R}}
	 		$
	 		is given by
	 				\begin{equation}
	 					\label{tracking_mdp_reward}
	 					{R\big((x,x^{\mathrm{d}},u^{\mathrm{d}}),u\big) := -J_{\manifold{X}}(x,x^{\mathrm{d}}) - J_{\manifold{U}}(u,u^{\mathrm{d}})},	 				\end{equation}
    \item and the transitions $				{{\transition} : \manifold{S} \times \manifold{A} \to \Delta(\manifold{S})}$ are defined by
\begin{gather}
\begin{gathered}
        x_{t+1} \sim f(\, \cdot \, |\, x_t, u_t),  
    \\ 
    x^{\mathrm{d}}_{t+1} \sim f(\, \cdot \, |\, x^{\mathrm{d}}_t, u^{\mathrm{d}}_t), \quad 
    u^{\mathrm{d}}_{t+1} \sim \rho.
\end{gathered}
    \label{transition_dynamics_sampled_from_version}
\end{gather}
\end{itemize}

\end{definition}

\begin{remark}
\label{justifying_mdp_formulation}
     This formulation allows us to model a tracking control problem over a broad class of reference trajectories (\textit{i.e.}, those generated by a certain stochastic process) as a single \textit{stationary} MDP (\textit{i.e.}, with time-invariant transitions and reward). While we could also formulate a (\textit{non-stationary}) MDP 
    corresponding to a \textit{particular} reference trajectory by making the tracking cost a function of time $t$ and the actual state $x$, an optimal policy for that MDP would be useless for tracking \textit{other} references.
 In {Sec. \ref{section:experiments}}, we will show empirically that policies trained in the proposed manner 
  also effectively track pre-planned reference trajectories, for which the sequence of reference actions ${\{u^{\mathrm{d}}_0, u^{\mathrm{d}}_1, u^{\mathrm{d}}_2, \cdots\}}$ is chosen to induce a pre-selected state trajectory ${\{x^{\mathrm{d}}_0, x^{\mathrm{d}}_1, x^{\mathrm{d}}_2, \cdots\}}$.
\end{remark}

\ifICRA
\begin{example}[name={\normalfont\texttt{Particle}\itshape },continues=example:particle]
	Following \eqref{transition_dynamics_sampled_from_version},
	the dynamics of $\manifold{M}_\manifold{T}$ for \texttt{Particle}
	can be expressed as%
	\begin{subequations}
 \label{point_particle_tracking_MDP_dynamics}
		\begin{align}
  \label{point_particle_tracking_MDP_dynamics_actual}
			r_{t+1} &= r_t + v_t \dt,
			\quad
			\hspace{2.4pt} v_{t+1} = v_t + \tfrac{1}{m} u_t \dt,
			\\
  \label{point_particle_tracking_MDP_dynamics_reference}
			r_{t+1}^{\mathrm{d}} &= r_t^{\mathrm{d}} + v_t^{\mathrm{d}} \dt,
			\quad
			v_{t+1}^{\mathrm{d}} = v_t^{\mathrm{d}} + \tfrac{1}{m} u_t^{\mathrm{d}} \dt,
			\\
			u_{t+1}^{\mathrm{d}} &\sim \mathcal{N}(0,\Sigma),
  \label{point_particle_tracking_MDP_dynamics_actions}
		\end{align}
	\end{subequations}
		where ${\big((r,v),(r^{\mathrm{d}}, v^{\mathrm{d}}), u^{\mathrm{d}}\big) \in \mathcal{S} = T\mathbb{R}^3 \times T\mathbb{R}^3 \times \mathbb{R}^3}$ and $u \in \mathcal{A} = \mathbb{R}^3$. From \eqref{tracking_mdp_reward}, the reward is given by
		\begin{equation}
			R(s,a)  = - \alpha({r - r^{\mathrm{d}}}) - c_v \hspace{.5pt}\norm{v - v^{\mathrm{d}}} 
   - c_u \hspace{.5pt} \norm{u - u^{\mathrm{d}}}.
			\label{point_particle_reward}
		\end{equation}
	
\end{example}
\fi

\subsection{Symmetries of Tracking Control MDPs}

We now show that the MDP induced by a tracking control problem with certain symmetries will inherit a related symmetry with certain convenient properties.

\begin{theorem}
	\label{tracking_problem_symmetry}
	Consider
	a tracking control problem
	${\manifold{T} = (\manifold{X}, \manifold{U}, \dynamics, J_{\manifold{X}}, J_{\manifold{U}}, \rho, \gamma)}$
	as well as Lie group actions 
		${\Upsilon: \manifold{K} \times \manifold{X} \to \manifold{X}}$ and ${\Theta: \manifold{H} \times \manifold{U} \to \manifold{U}}$. Suppose
that:
	\begin{itemize}
		\item 
		$J_{\manifold{X}}$ is $\Upsilon$-invariant and $J_{\manifold{U}}$ is $\Theta$-invariant, {\normalfont i.e.},  
				for all ${x,x^{\mathrm{d}} \in \manifold{X}}$, ${u,u^{\mathrm{d}} \in \manifold{U}}$, $k \in \manifold{K}$, and ${h \in \manifold{H}}$, we have
		\begin{subequations}
  \label{running_cost_invariance}
			\begin{align}
				J_{\manifold{X}}(x,x^{\mathrm{d}}) &= J_{\manifold{X}}\big(\Upsilon_k (x), \Upsilon_k (x^{\mathrm{d}})\big),
				\label{tracking_cost_invariance}
				\\ 
				J_{\manifold{U}}(u,u^{\mathrm{d}}) &= J_{\manifold{U}}\big(\Theta_h (u), \Theta_h (u^{\mathrm{d}})\big).
				\label{effort_cost_invariance}
			\end{align}
		\end{subequations}
			\item For each ${(k,h) \in \manifold{K} \times \manifold{H}}$, there exists ${{k^\prime} \in \manifold{K}}$ such that
			for all $(x,u) \in \manifold{X} \times \manifold{U}$ and $B \in \mathsf{B}(\manifold{X})$, we have
			\begin{equation}
				\dynamics\big(
				\Upsilon_{{k^\prime}}(B) \,|\, x, u
				\big)
				=
				\dynamics\big(
				B  \,|\, \Upsilon_{k}(x), \Psi_{h}(u)	
				\big).
				\label{physical_system_invariance}
			\end{equation}
		\end{itemize} 
			Define actions of the direct product group ${\manifold{G} = \manifold{K} \times \manifold{H}}$ on ${\manifold{S} = \manifold{X} \times \manifold{X} \times \manifold{U}}$ and ${\manifold{A} = \manifold{U}}$, given respectively by%
			\begin{subequations}
\label{combined_group_action_definition}
\begin{gather}
						\label{combined_group_action_definition_state}
						\Phi_{(k,h)}(x,x^{\mathrm{d}},u^{\mathrm{d}}) := \big(\Upsilon_k(x), \Upsilon_k(x^{\mathrm{d}}), \Theta_h(u^{\mathrm{d}})\big),
						\\
						\Psi_{(k,h)}(u) := \Theta_h(u).
				\label{combined_group_action_definition_actions}					\end{gather}
			\end{subequations}
			Then, $(\Phi, \Psi)$ is a Lie group symmetry of $\manifold{M}_{\manifold{T}}$. Moreover, if $\Upsilon$ and $\Theta$ are free and proper, then $\Phi$ is also free and proper.

			\end{theorem}
			
\ifICRA			
			\begin{proof}
				From \eqref{tracking_mdp_reward}, we compute the transformed reward as
				\begin{align}
					&{\reward}\big(\Phi_{(k,h)}(s),\Psi_{(k,h)}(a)\big) 
					\\ &= -J_{\manifold{X}}\big(\Upsilon_k(x),\Upsilon_k(x^{\mathrm{d}})\big) - J_{\manifold{U}}\big(\Theta_h(u),\Theta_h(u^{\mathrm{d}})\big)
					\\ &= -J_{\manifold{X}}(x,x^{\mathrm{d}}) - J_{\manifold{U}}(u,u^{\mathrm{d}})
					= R(s,a),
				\end{align}
				where we have substituted in \eqref{combined_group_action_definition} and simplified using
				\eqref{running_cost_invariance}. Thus, \eqref{reward_invariance_lie_group_symmetry} holds.
    Considering now the transitions, we note that \eqref{transition_dynamics_sampled_from_version} can also be written using the ``product measure'' as 
    \begin{equation}
	\label{tracking_mdp_dynamics}
					\transition\big(
	\cdot |\, 
	(x,x^{\mathrm{d}},u^{\mathrm{d}}),u
	\big) 
	:= 
	\dynamics(\, \cdot \, | \, x,u) \times 
	\dynamics(\, \cdot \, | \, x^{\mathrm{d}},u^{\mathrm{d}}) \times 
	\rho.
\end{equation}%
    We then apply \eqref{combined_group_action_definition} to \eqref{tracking_mdp_dynamics} to compute
				\begin{align}
					\nonumber
					&\transition\big(\cdot|\, \Phi_{(k,h)}(s), \Psi_{(k,h)}(a)\big) 
					\\ &= 
					\dynamics\big(\cdot|\,\Upsilon_k(x), \Theta_h(u)\big)
					\hspace{-.5pt}\times\hspace{-.5pt}
					\dynamics\big(\cdot|\,\Upsilon_k(x^{\mathrm{d}}), \Theta_h(u^{\mathrm{d}})\big)
					\hspace{-.5pt}\times\hspace{-.5pt}
					\rho
					\\ &= 
					\dynamics\big(\Upsilon_{{k^\prime}}(\,\cdot\,) |\,x,u\big)
					\times
					\dynamics\big(\Upsilon_{{k^\prime}}(\,\cdot\,) |\,x^{\mathrm{d}},u^{\mathrm{d}}\big)
					\times \rho,
					\label{switch_transitions_to_h}
				\end{align}
				where \eqref{switch_transitions_to_h} follows from \eqref{physical_system_invariance}.
				Considering any $\Phi$-invariant ${B \in \mathsf{B}(\manifold{S})}$, we note that ${B = \Phi_{{({k^\prime}{}^{-1},\identity_{\manifold{H}})}}(B)}$, and compute
				\begin{align}
					\nonumber
					&\transition
					\big(B \,|\, \Phi_{(k,h)}(s), \Psi_{(k,h)}(a)\big) 
					\\ &= \transition\big(\Phi_{{({k^\prime}{}^{-1},\identity_{\manifold{H}})}}(B) \,|\, \Phi_{(k,h)}(s), \Psi_{(k,h)}(a)\big) 
					\\ &= \big(
					\dynamics(\,\cdot\, |\,x,u)
					\times
					\dynamics(\,\cdot\, |\,x^{\mathrm{d}},u^{\mathrm{d}})
					\times  \rho
					\big)(B)
					= \transition\big(B \,|\, s,a\big),
					\label{transition_punch_line}
				\end{align}
				where \eqref{transition_punch_line} follows directly from \eqref{switch_transitions_to_h} and \eqref{combined_group_action_definition_state}. Thus, \eqref{transition_invariance_lie_group_symmetry} holds as well, and $(\Phi, \Psi)$ is a Lie group symmetry of $\manifold{M}_{\manifold{T}}$.
Assuming that $\Upsilon$ and $\Theta$ are free and proper,
it is readily verified that  $\Phi$ is free and proper after noting that
  $\Phi$ is the product action of $\Gamma$ and $\Theta$ (\textit{i.e.}, ${\Phi_{(k,h)} = \Gamma_k \times \Theta_h}$), where $\Gamma$ is the diagonal action of $\Upsilon$ (\textit{i.e.}, ${\Gamma_k = \Upsilon_k \times \Upsilon_k}$). 
			\end{proof}

   \fi

\begin{remark}
Because we do \textit{not} assume that ${k^\prime = k}$,
 	\eqref{physical_system_invariance} is more general than  equivariance of the transitions. However, $k^\prime$ must depend only on $k$ and $h$, and not on $x$ and $u$.
\end{remark}

\ifICRA
\begin{example}[name={\normalfont\texttt{Particle}\itshape },continues=example:particle]
	Considering the Lie groups ${\manifold{K} := T\mathbb{R}^3}$ (with the obvious group operation inherited from its identification with ${\mathbb{R}^3 \times \mathbb{R}^3}$) and ${\manifold{H} := \mathbb{R}^3}$, we let a $\manifold{K}$-action on ${\manifold{S}=T\mathbb{R}^3}$ and an $\manifold{H}$-action on ${\manifold{A} = \mathbb{R}^3}$ be given by the left action of the groups on themselves, \textit{i.e.},
	\begin{subequations}
 \label{point_particle_actions}
		\begin{align}
			\label{point_particle_x_action}
			\Upsilon_{k}(r,v) &:= L_{(k_1,k_2)}(r,v) = (r+k_1,v+k_2), 
			\\ 
			\label{point_particle_u_action}
			\Theta_{h}(u) &:= L_h(u) = u + h,
		\end{align}
	\end{subequations}
	which are free and proper.
	It is clear that
	the tracking and effort costs \eqref{point_particle_running_costs} are invariant to these actions,
	\textit{i.e.},  \eqref{running_cost_invariance} holds. Moreover, for any ${B \in \mathsf{B}(T\mathbb{R}^3)}$, 
	\begin{align}
 \nonumber
		&f\big(B \,|\,\Upsilon_k(x),\Theta_h(u)\big)
            \\ &
            =f\big(B\,|\,(r + k_1, v + k_2),u + h\big)
		\\ &=
		\begin{cases}
			1, \ 
			\begin{pmatrix}
				r + k_1 + (v + k_2) \dt
				\\
				v + k_2 + \tfrac{1}{m}(f + h)\dt 
			\end{pmatrix} \in B,
			\\ 0, \ \textrm{otherwise}.
		\end{cases}
		\\ &=
		f\big(\Upsilon_{{k^\prime}}(B) \,|\, x,u\big), \ \ 
		k^\prime = -(k_1,k_2) - (k_2,\tfrac{1}{m}h) \dt.
	\end{align}
	Thus, the transitions satisfy \eqref{physical_system_invariance}.
	In the manner of \eqref{combined_group_action_definition}, the group actions \eqref{point_particle_actions} induce actions of ${\manifold{G} = \manifold{K} \times \manifold{H} = T\mathbb{R}^3 \times \mathbb{R}^3} $ on $\mathcal{S}$ and $\mathcal{A}$, given by
	\begin{subequations}
 \label{point_particle_symmetry_direct_product}
		\begin{align}
			\Phi_{(k,h)}(s) &:= 
			(r,v,r^{\mathrm{d}}, v^{\mathrm{d}}, u^{\mathrm{d}}) + (k_1,k_2,k_1,k_2,h), 
                \label{point_particle_state_symmetry_direct_product}
			\\
			\Psi_{(k,h)}(a) &:= u^{\mathrm{d}} + h.
                \label{point_particle_action_symmetry_direct_product}
		\end{align}
	\end{subequations}
	Thus, by Theorem \ref{tracking_problem_symmetry}, $(\Phi,\Psi)$ is a Lie group symmetry of $\mathcal{M}_{\mathcal{T}}$ for the \texttt{Particle}, and moreover, $\Phi$ is free and proper.
\end{example}
\fi

\section{Continuous MDP Homomorphisms \\ Induced by Lie Group Symmetries}

\label{section:mdp_homomorphisms}

\label{section_MDPH}

We will use the following theorem to show that symmetries of a tracking control MDP can be used to reduce its dimension via a homomorphism and also give an explicit formula for policy lifting. 
Although related results are known in the discrete \cite{Ravindran2004} and deterministic \cite{Yu2023} settings, we require a more general result due to our continuous state and action spaces and the random sampling of the reference actions (even when the underlying dynamics $f$ are deterministic).

\begin{theorem}
	\label{symmetry_implies_homomorphism}
	Consider an MDP 
		 	${\manifold{M} = (\manifold{S},\manifold{A},\reward,\transition,\gamma)}$
	 with a Lie group symmetry $(\Phi,\Psi)$.
	Suppose that
		$\Phi$ is free and proper and 
		${\lambda : \manifold{S} \to 
			\manifold{G}}$ is 
   \ifICRA
   any\footnote{
    Since $\lambda$ need not be continuous, it can be constructed from a collection of \textit{local} trivializations of the principal $\mathcal{G}$-bundle ${p : \manifold{S} \to \manifold{S} / \manifold{G}}$ \cite[\S 9.9]{GallierII}.
}
\fi   
   equivariant map.
 Define 
	\begin{subequations}
		\begin{gather}
			p : S \to \manifold{S}/\manifold{G}, \ s \mapsto \Phi_G(s),
			\label{state_map_mdp_homomorphism_theorem}
			\\
			h : \manifold{S} \times \manifold{A} \to \manifold{A}, \ (s,a) \mapsto \Psi_{\lambda(s)^{-1}}(a).
			\label{action_map_mdp_homomorphism_theorem}
		\end{gather}
	\end{subequations}
	Then, $(p,h)$ is an MDP homomorphism from $\manifold{M}$ to
	${
		\widetilde{\manifold{M}} = \big(\widetilde{\manifold{S}} = \manifold{S} / \manifold{G}, \widetilde{\manifold{A}} = \manifold{A} , \widetilde{R},\widetilde{\transition},\gamma\big)
	}$, where we define
	\begin{subequations}
		\begin{align}
			\widetilde{R}(\tilde{s},\tilde{a}) &:= R\big(s,\Psi_{\lambda(s)}
			(\tilde{a})\big) \,\big|\, {}_{s \, \in \,  p^{-1}(\tilde{s})},
			\label{reduced_reward_definition}
			\\
			\widetilde{\transition}(\widetilde{B} \,|\, \tilde{s},\tilde{a}) &:=  \transition\big(
			p^{-1}(\widetilde{B}) \,|\,
			s,\Psi_{\lambda(s)}(\tilde{a}) \big) \,\big|\, {}_{s \, \in \,  p^{-1}(\tilde{s})}
			\label{reduced_transition_dynamics_definition}
		\end{align}
	\end{subequations}
	 independent of the particular choice of $s$.
	Also, for any policy 
	${\widetilde{\pi}}$ 
	for $\widetilde{\manifold{M}}$, 
a policy for $\manifold{M}$
that is 
a lift of $\widetilde{\pi}$ is
	given by
	\begin{equation}
		(\widetilde{\pi})^\uparrow (A \,|\, s) := \widetilde{\pi}\big(
		\Psi_{\lambda(s)^{-1}}(A) \,|\, p(s)
		\big).
            \label{define_lifted_policy}
	\end{equation}

\end{theorem}

\ifICRA
\begin{proof}
	Because $\Phi$ is free and proper, 
	${\manifold{S}/\manifold{G}}$ is a smooth manifold 
 of dimension $\dim \manifold{S} - \dim \manifold{G}$  \cite[Thm. 21.10]{SmoothLee}.  
	We first verify that $\widetilde{R}$ and $\widetilde{\transition}$  
 are well-defined (\textit{i.e.}, their values do not depend on the particular choice of ${s \in p^{-1}(\tilde{s})}$).
Since $p$ maps states to $\Phi$-orbits, for any ${s_1, s_2 \in p^{-1}(\tilde{s})}$, there exists some ${g \in \manifold{G}}$ such that ${s_1 = \Phi_g(s_2)}$. Thus, following \eqref{reduced_reward_definition}, 
	\begin{align}
		R(\tilde{s},\tilde{a})
		&=
		R\big(s_1,\Psi_{\lambda(s_1)}(\tilde{a})\big) 
		\label{same_orbit_phase_offset}
		\\		&=
R\big( \Phi_g(s_2),\Psi_{g \lambda (s_2)}(
\tilde{a})\big) 
\label{using_equivariant_map}	
\\&=
		R\big( s_2,\Psi_{\lambda (s_2)}(\tilde{a})\big)
= 		R(\tilde{s},\tilde{a}),
\end{align}	
	where \eqref{using_equivariant_map} follows from the equivariance of $\lambda$ and 
	the invariance of the reward. Similarly, from \eqref{reduced_transition_dynamics_definition}, we compute
	\begin{align}
 \hspace{-2pt}
		\widetilde{\transition}(\widetilde{B} \,|\, \tilde{s},\tilde{a}) 
		&= \transition\big(
		p^{-1}(\widetilde{B})
		\,|\,
		s_1,\Psi_{\lambda(s_1)}(\tilde{a})\big) 
		\\		&=
		\transition\big( p^{-1}(\widetilde{B}) \,|\, \Phi_g(s_2),\Psi_{g\lambda (s_2)}(
		\tilde{a})\big) 
		\\		&=
		\transition\big( 
		p^{-1}(\widetilde{B})
		\,|\, s_2,\Psi_{\lambda (s_2)}(
		\tilde{a})\big)
=		\widetilde{\transition}(\widetilde{B} \,|\, \tilde{s},\tilde{a}) ,
		\label{same_transition_s2}
	\end{align}
	where \eqref{same_transition_s2} follows from \eqref{transition_invariance_lie_group_symmetry}, since for any ${\widetilde{B} \in \mathsf{B}({\widetilde{S}})}$, ${p^{-1}(\widetilde{B}) \in \mathsf{B}(\manifold{S})}$ is a $\Phi$-invariant Borel set.

We now verify the MDP homomorphism. 
Since for each ${g \in \manifold{G}}$, the map $\Psi_g$ is a diffeomorphism, $h_s$ is measurable and surjective for each ${s \in \manifold{S}}$. On the other hand, $p$ is surjective by construction and measurable because orbits of proper actions are closed 
\cite[Cor. 21.8]{SmoothLee}.
Since 
${s \in p^{-1}\big(p(s)\big)}$, 
	\begin{align}
		\nonumber \widetilde{R}\big(p(s),h(s,a)\big)
		&= R\big(s,\Psi_{
			\lambda(s) }
		\circ  
		\Psi_{\lambda(s)^{-1}}(a)
		\big) = R(s,a),
	\end{align}
	hence \eqref{reward_invariance_mdp_homomorphism} holds.
	We verify \eqref{transition_invariance_mdp_homomorphism} similarly, since by \eqref{reduced_transition_dynamics_definition}, 
	\begin{align}
		\nonumber 
		\widetilde{\transition}\big(\widetilde{B} \,|&\, p(s),h(s,a)\big) 
		\\ &= \transition\big(
		p^{-1}(\widetilde{B}) \,|\,
		s,\Psi_{\lambda(s)} \circ 
		\Psi_{\lambda(s)^{-1}}(a)\big)
		\\ &= \transition\big(
		p^{-1}(\widetilde{B})
		\,|\, s,a\big).
	\end{align}
	Thus, $(p,h)$ is an MDP homomorphism from $\manifold{M}$ to $\widetilde{\manifold{M}}$.
 Finally, to see that $(\widetilde{\pi})^\uparrow$ is a lift of $\widetilde{\pi}$,
we compute
\begin{align}
    (\widetilde{\pi})^\uparrow \big(h_s^{-1}(\widetilde{A}) \,|\, s\big) 
    &=
        (\widetilde{\pi})^\uparrow \big(\Psi_{\lambda(s)}(\widetilde{A}) \,|\, s\big) 
        \label{using_definition_of_action_map}
 \\   &= \widetilde{\pi}\big(
    \Psi_{\lambda(s)^{-1}} \circ \Psi_{\lambda(s)}(\widetilde{A}) \,|\, p(s)
    \big) \label{using_definition_of_lifted_policy}
 \\   &= \widetilde{\pi}\big(\widetilde{A} \,|\, p(s)
    \big),
    \end{align}
    where \eqref{using_definition_of_action_map} and \eqref{using_definition_of_lifted_policy} follow directly from   \eqref{action_map_mdp_homomorphism_theorem}, \eqref{define_lifted_policy}, and the fact that $\Psi_g$ is a diffeomorphism for all $g \in \manifold{G}$.
 \end{proof}

 \fi
 \section{Quotient MDPs for Tracking Control \\ in Free-Flying Robotic Systems}

\label{section:application}

\ifICRA
It is now clear that Theorems \ref{mdp_homomorphism_value_equivalence_and_optimality},  \ref{tracking_problem_symmetry}, and \ref{symmetry_implies_homomorphism} can be applied together to reduce the MDP induced by a tracking control problem with symmetry in its dynamics and running costs.
\newpage
\else
\fi

\ifICRA

\begin{example}[name={\normalfont\texttt{Particle}\itshape },continues=example:particle]
Using the symmetry \eqref{point_particle_symmetry_direct_product} of 
$\manifold{M}_{\manifold{T}}$ for the
\texttt{Particle}, we will construct an MDP homomorphism using Theorem \ref{symmetry_implies_homomorphism}.  
	Recall that the state of $\manifold{M}_{\manifold{T}}$ is ${s = \big((r,v),(r^{\mathrm{d}},v^{\mathrm{d}}),u^{\mathrm{d}}\big)}$. We 
	first define 
	\begin{align}
		\lambda
		 (s) &
   := 
   \big((r^{\mathrm{d}},v^{\mathrm{d}}),u^{\mathrm{d}}\big),
  \label{point_particle_lambda_map}
		\\
		p 
  (s) &
  :=
  (r - r^{\mathrm{d}}, v - v^{\mathrm{d}}).
  \label{point_particle_observation}
	\end{align}
	It is easily verified that $\lambda$ is equivariant and $p$ maps each state $s$ to its $\Phi$-orbit.
We now define a quotient MDP 
		$\widetilde{\manifold{M}_{\manifold{T}}}$
 as described in Theorem \ref{symmetry_implies_homomorphism}. 
 The state of $\widetilde{\manifold{M}_{\manifold{T}}}$ is ${\tilde{s} = (r^{\mathrm{e}}, v^{\mathrm{e}}) \in \widetilde{\manifold{S}} = \manifold{S} / \manifold{G}  \simeq T \mathbb{R}^3}$ 
and the actions are ${\tilde{a} = u^{\mathrm{e}} \in \widetilde{A} = \mathbb{R}^3}$. From \eqref{action_map_mdp_homomorphism_theorem} and \eqref{point_particle_lambda_map}, we may derive  
\begin{equation}
    h(s,a) = \Psi_{(-r^{\mathrm{d}},-v^{\mathrm{d}},-u^{\mathrm{d}})}(u) = u - u^{\mathrm{d}}.
\end{equation}
Since clearly ${\big((r^{\mathrm{e}}, v^{\mathrm{e}}), (0, 0), 0\big) \in p^{-1}(r^{\mathrm{e}},v^{\mathrm{e}})}$, from \eqref{reduced_reward_definition}, \eqref{point_particle_action_symmetry_direct_product}, and \eqref{point_particle_lambda_map} we may construct the reduced reward as
\begin{align}
&\widetilde{R}(\tilde{s},\tilde{a}) = R\big(s,\Psi_{\lambda(s)}(\tilde{a})\big) \,\Big|\, {}_{s=\mathlarger((r^{\mathrm{e}}, v^{\mathrm{e}}), (0, 0), 0\mathlarger)}
\\& =
- \alpha({r^{\mathrm{e}}}) - c_v \hspace{.5pt}\norm{v^{\mathrm{e}}}
-c_u \hspace{.5pt} \norm{u^{\mathrm{e}}}.
\end{align} 
Likewise, a straightforward calculation using \eqref{point_particle_tracking_MDP_dynamics}, \eqref{point_particle_action_symmetry_direct_product}, and \eqref{reduced_transition_dynamics_definition} will show that the reduced transitions are  given by
	\begin{align}
 \nonumber
		\widetilde{\transition}(\widetilde{B} \,|\, \tilde{s},\tilde{a}) &=  \transition\big(
		p^{-1}(\widetilde{B}) \,|\,
		s,\Psi_{\lambda(s)}(\tilde{a}) \big)
 \,\Big|\, {}_{s=\mathlarger((r^{\mathrm{e}}, v^{\mathrm{e}}), (0, 0), 0\mathlarger)}
\\ &= \begin{cases}
1, \, 
(r^{\mathrm{e}} + v^{\mathrm{e}} \dt, v^{\mathrm{e}} + \tfrac{1}{m} u^{\mathrm{e}} \dt)
\in \widetilde{B},
\\ 0, \, \textrm{otherwise},
\end{cases}
\end{align}
which is nothing but the usual ``error dynamics'' \cite{Maithripala2006}, \textit{i.e.}%
\begin{align}
		r_{t+1}^{\mathrm{e}} = r_t^{\mathrm{e}} + v_t^{\mathrm{e}} \dt, \quad v_{t+1}^{\mathrm{e}} = v_t^{\mathrm{e}} + \tfrac{1}{m} u_t^{\mathrm{e}} \dt.
\end{align}
Finally, by Theorem \ref{symmetry_implies_homomorphism}, $(p,h)$ is an MDP homomorphism from $\manifold{M}_{\manifold{T}}$
	to $\widetilde{\manifold{M}_{\manifold{T}}} = 
	\big(
 T\mathbb{R}^3,
\mathbb{R}^3,
	\widetilde{R}, \widetilde{\tau}, \gamma
	\big)
	$, and
moreover we may lift any policy $\widetilde{\pi}$ for $\widetilde{\manifold{M}_{\manifold{T}}}$
 to $\widetilde{\manifold{M}}$ using \eqref{define_lifted_policy}, obtaining
 \begin{align}
    (\widetilde{\pi})^\uparrow (A \,|\, s) 
  &= \widetilde{\pi}(
  A - u^{\mathrm{d}} \,|\,
  r - r^{\mathrm{d}}, v - v^{\mathrm{d}}).
  \label{particle_lifted_policy}
\end{align}
By Theorem \ref{mdp_homomorphism_value_equivalence_and_optimality}, the action-value function for $(\widetilde{\pi})^\uparrow$ satisfies
\begin{align}
    Q^{(\widetilde{\pi})^\uparrow}(s,a) 
    &= \widetilde{Q}^{\widetilde{\pi}}\big((r - r^{\mathrm{d}}, v - v^{\mathrm{d}}),u - u^{\mathrm{d}}\big).
  \label{particle_action_value_function}
\end{align}
Thus, an optimal policy can observe only the position and velocity error and augment the result with the reference force (\textit{i.e.}, $u = \tilde{\pi}(r-r^{\mathrm{d}}, v-v^{\mathrm{d}}) + u^{\mathrm{d}}$ for deterministic policies).
\end{example}

\begin{example}[{\normalfont \texttt{Astrobee}}\textit {\cite{bualat2015astrobee}}]
\label{astrobee_example}
\begingroup
	This space robot has state ${x = (q,\xi)}$ in ${\manifold{X} = SE(3) \times \mathbb{R}^6}$ (\textit{i.e.}, the pose $q$ as a homogeneous transform 
 and twist ${\xi = (\omega,v)}$) and action ${u = (\mu, f)}$ in ${ \manifold{U} = \mathbb{R}^6}$ (\textit{i.e.}, the applied wrench).
 The dynamics are%
\begin{subequations}
\label{astrobee_dynamics}
    \begin{align}
    \label{astrobee_kinematics}
        q_{t+1} &= q_t \exp (\hat{\xi}_t \dt),    
\\
v_{t+1} &= v_t + \tfrac{1}{m} f_t \dt,
    \label{astrobee_newton_equations}
\\
\omega_{t+1} &= \omega_t + \mathbb{J}^{-1} (\mu_t - \omega_t \times \mathbb{J} \, \omega_t) \dt,
    \label{astrobee_euler_equations}
    \end{align}
\end{subequations}
where $\hat{\cdot} : \mathbb{R}^6 \to \mathfrak{se}(3)$.
The running costs are defined by %
	\begin{subequations}
  \label{astrobee_costs}
		\begin{align}
  &
  \begin{aligned}
      			& J_{\manifold{X}}(x,x^{\mathrm{d}}) := 
      \\    & \quad \quad 
         \alpha(r - r^{\mathrm{d}}) + 
         c_R \norm{
   \log
   (R^{\mathrm{T}} R^{\mathrm{d}})} \, + 
   c_\xi \norm{\xi - \xi^{\mathrm{d}}}
   ,
  \end{aligned}
\label{astrobee_tracking_cost}
			\\
			&J_{\manifold{U}}(u,u^{\mathrm{d}}) := c_u \norm{u - u^{\mathrm{d}}},
		\end{align}
	\end{subequations}
 where $r$ and $R$ are the $\mathbb{R}^3$ and $SO(3)$ components of $q$.
Letting ${\rho = \mathcal{N}(0,\Sigma)}$,  we may construct $\mathcal{M}_{\mathcal{T}}$ as in Def. \ref{tracking_mdp_definition}.
Next, let ${\mathcal{K} = SE(3)}$ act on $\mathcal{X}$ and ${\mathcal{H} = \{\identity\}}$ act on $\manifold{U}$ via
    \begin{align}
        \Psi_{k}(q, \xi) := ( k q, \xi),
        \quad
        \Theta_{h}(w) := w,
        \label{astrobee_actions}
    \end{align}
where \eqref{running_cost_invariance} and \eqref{physical_system_invariance} hold for these free and proper actions. Using Theorem \ref{tracking_problem_symmetry} to derive a symmetry of 
$\mathcal{M}_{\mathcal{T}}$ as in \eqref{combined_group_action_definition}, we apply
 Theorem \ref{symmetry_implies_homomorphism} with ${\lambda : 
s \mapsto q}$ to ultimately obtain an MDP homomorphism $(p,h)$, where
 ${h_s = \id}$ for all ${s = (q, \xi, q^{\mathrm{d}}, \xi^{\mathrm{d}}, u^{\mathrm{d}}) \in \mathcal{S}}$, and
\begin{gather}
\label{astrobee_observation}
    p
    (s) 
    :=
    \big(q^{-1} q^{\mathrm{d}}, \xi, \xi^{\mathrm{d}}, u^{\mathrm{d}}\big).
\end{gather}
Thus, an optimal policy and its $Q$ function 
can be written
\begin{align}
\label{astrobee_lifted_policy}
        (\widetilde{\pi})^\uparrow (A \,|\, s) &= 
        \widetilde{\pi}\big( A \,|\, 
        (q^{-1} q^{\mathrm{d}}, \xi, \xi^{\mathrm{d}}, u^{\mathrm{d}})\big),
        \\
\label{astrobee_action_value_function}
    Q^{(\widetilde{\pi})^\uparrow}(s,a) 
    &= \widetilde{Q}^{\widetilde{\pi}}\big(
    (q^{-1} q^{\mathrm{d}}, \xi, \xi^{\mathrm{d}}, u^{\mathrm{d}})
    ,u\big).
\end{align}
Hence, an optimal policy can be learned using an observation that sees only the \textit{error} between the actual and reference poses, instead of observing these poses separately.

\endgroup
\end{example}

\begin{example}[{\normalfont \texttt{Quadrotor}}\textit {\cite{Mellinger2011}}]
\label{quadrotor_example}
This aerial robot has the same state space as the \texttt{Astrobee}, but the actions are the ``single-rotor thrusts''
${u \in \manifold{U} = \mathbb{R}^4}$. The dynamics and running costs are as given in 
\eqref{astrobee_dynamics} and \eqref{astrobee_costs}, but with
${f_t = (0,0, 
u^1_t + u^2_t + u^3_t + u^4_t
)}$
and
${\mu_t = \big(\ell (u^1_t - u^3_t), 
\ell (u^2_t - u^4_t), 
c (u^1_t - u^2_t + u^3_t - u^4_t)\big)}$, and%
\begin{align}
    v_{t+1} &= v_t + \big(\tfrac{1}{m} f_t - R_t{}^{\textrm{T}} (\mathrm{g} \, \mathrm{e}_3) \big)\dt
\end{align}
instead of \eqref{astrobee_newton_equations}, where ${\mathrm{g}}$ is the magnitude of gravitational acceleration, ${\mathrm{e}_3 = (0,0,1)}$, and ${R_t \in SO(3)}$ is the rotation component of $q_t$. Gravity ``breaks'' the $SE(3)$ symmetry of the system, but preserves the $SE(3)$ subgroup
\begingroup
\setlength\arraycolsep{3pt}
\begin{align}
    \mathcal{K}^\prime = \left\{
    \begin{pmatrix}
        \mathrm{rot}_z(\theta) & r \\ 0 & 1
    \end{pmatrix} : (r, \theta) \in \mathbb{R}^3 \times \mathbb{S}^1
    \right\}
\end{align}
\endgroup
which is isomorphic (as a Lie group) to ${SE(2) \times \mathbb{R}}$ and acts on ${SE(3) \times \mathbb{R}^6}$ via the restriction of \eqref{astrobee_actions}. Using Theorems \ref{tracking_problem_symmetry} and \ref{symmetry_implies_homomorphism}, we may derive an MDP homomorphism $(p,h)$ for which
 $h_s = \id$ for all $s = (q, \xi, q^{\mathrm{d}}, \xi^{\mathrm{d}}, u^{\mathrm{d}}) \in \mathcal{S}$ and
\begin{gather}
    p
    (s) 
    :=
    (q^{-1} q^{\mathrm{d}}, R^{\,\textrm{T}} \mathrm{e}_3,  \xi, \xi^{\mathrm{d}}, u^{\mathrm{d}}),
    \label{quadrotor_observation}
\end{gather}
noting that $R^{\,\textrm{T}} \mathrm{e}_3$ is the gravity direction in body coordinates.
Thus, an optimal policy (and its $Q$ function) can be written
\begin{align}
\label{quadrotor_lifted_policy}
        (\widetilde{\pi})^\uparrow (A \,|\, s) &= 
        \widetilde{\pi}\big( A \,|\, 
(q^{-1} q^{\mathrm{d}}, R^{\,\textrm{T}} \mathrm{e}_3,  \xi, \xi^{\mathrm{d}}, u^{\mathrm{d}})
        \big),
        \\
\label{quadrotor_action_value_function}
    Q^{(\widetilde{\pi})^\uparrow}(s,a) 
    &= \widetilde{Q}^{\widetilde{\pi}}\big(
(q^{-1} q^{\mathrm{d}}, R^{\,\textrm{T}} \mathrm{e}_3,  \xi, \xi^{\mathrm{d}}, u^{\mathrm{d}})
,u\big).
\end{align}
Hence, our theory demonstrates that for quadrotors, the state space of the tracking problem can be reduced by replacing the reference and actual poses with the pose error and the body-frame gravity vector, without degrading the best-case learned policy. Consider how this differs from heuristic approximations in prior work such as \cite{Molchanov2019}, whose state included the entire orientation $R$ (incompletely reducing the symmetry) and replaced the actual and reference angular velocities with the velocity error, which corresponds to an \textit{approximate} symmetry due to the ``cross terms'' in \eqref{astrobee_euler_equations}.

\end{example}

\fi

\section{Experiments} \label{section:experiments}

We now explore the effects of our symmetry-informed approach on sample efficiency and performance of model-free reinforcement learning for tracking control. RL environments were implemented for each of the tracking control MDPs in Examples \ref{particle_example}-\ref{quadrotor_example}, written in \texttt{jax} \cite{jax2018github} for performance. To implement environments for the quotient MDP arising from reduction by a symmetry group, we modify each environment's observation to the reduced state given in \eqref{point_particle_observation}, \eqref{astrobee_observation}, and \eqref{quadrotor_observation} (whereas the baseline sees the full-state observation $(x,x^{\mathrm{d}}, u^{\mathrm{d}})$). 
\ifICRA
As indicated respectively by \eqref{particle_lifted_policy}-\eqref{particle_action_value_function}, \eqref{astrobee_lifted_policy}-\eqref{astrobee_action_value_function}, and 
\eqref{quadrotor_lifted_policy}-\eqref{quadrotor_action_value_function}, we also modify (\textit{i.e}, lift) the actions generated by the learned policy (and those passed to the action-value function). 
\fi
For the \texttt{Particle} environment, we isolate the effects of reduction by different subgroups of the symmetry 
\ifICRA
given in \eqref{point_particle_symmetry_direct_product}
\fi
by also implementing environments reduced by translational symmetry alone (\textit{i.e.}, ${p(s) := (r - r^{\mathrm{d}}, v, v^{\mathrm{d}}, u^{\mathrm{d}})}$)
and by translational and velocity symmetry  alone (\textit{i.e.}, ${p(s) := (r - r^{\mathrm{d}}, v - v^{\mathrm{d}}, u^{\mathrm{d}})}$).

We use a custom implementation of PPO \cite{schulman2017proximal} (see code for details), with the same hyperparameters across all variants of each environment.
During training, the reference actions are sampled from a stationary distribution (as in Def. \ref{tracking_mdp_definition}), but we evaluate zero-shot on pre-planned (dynamically feasible) reference trajectories. 
Fig. \ref{fig:all_envs}
and Table \ref{table:all_envs}
report total reward (during training) and average tracking error (during evaluation) when starting from a randomly sampled initial state.

\begin{figure}[t]
\centering
\includegraphics[trim={0.25cm 0.10cm 0.25cm 0},clip,width=\columnwidth]{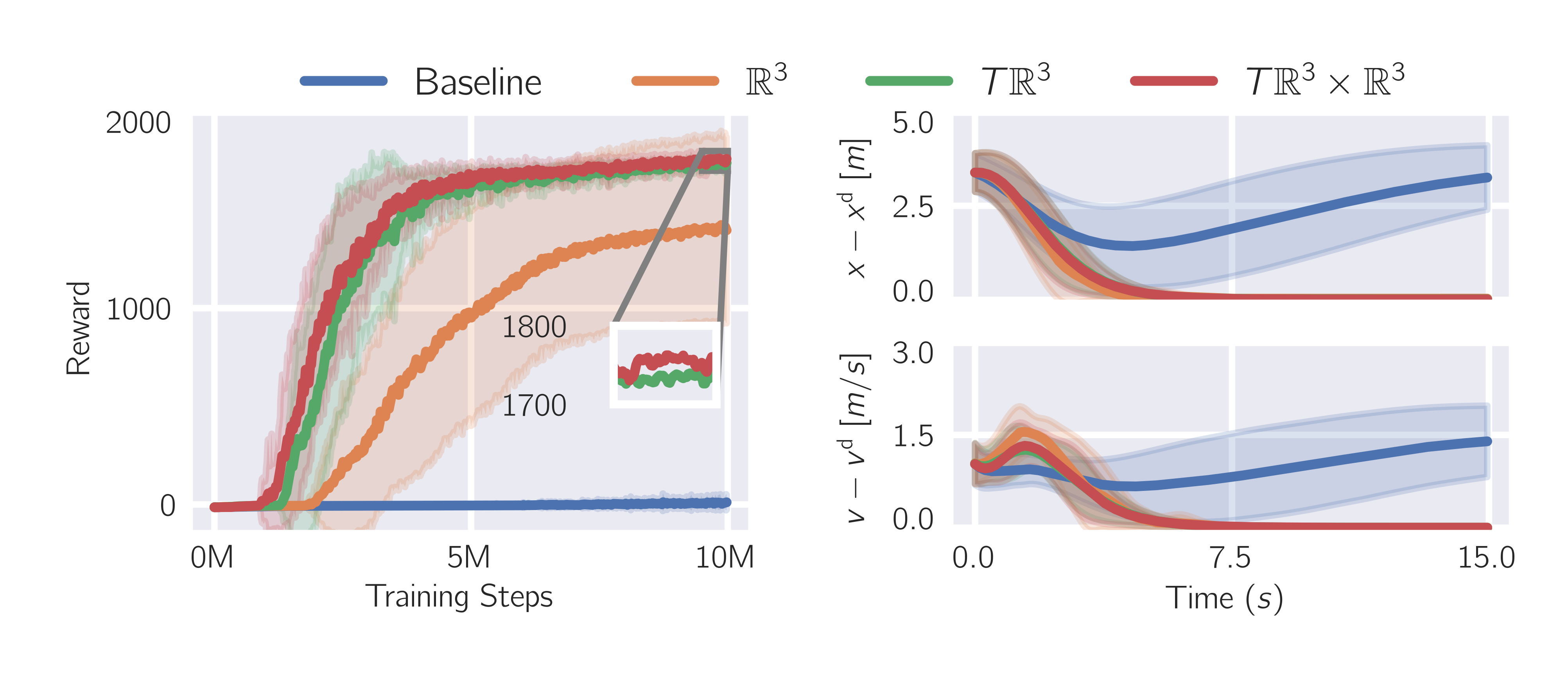}

\vspace{-10pt}
{\footnotesize (a)
Training and evaluation for
\texttt{Particle}%
.
}
\vspace{0pt}

\includegraphics[trim={0.25cm 0.10cm 0.25cm 0},clip,width=\columnwidth]{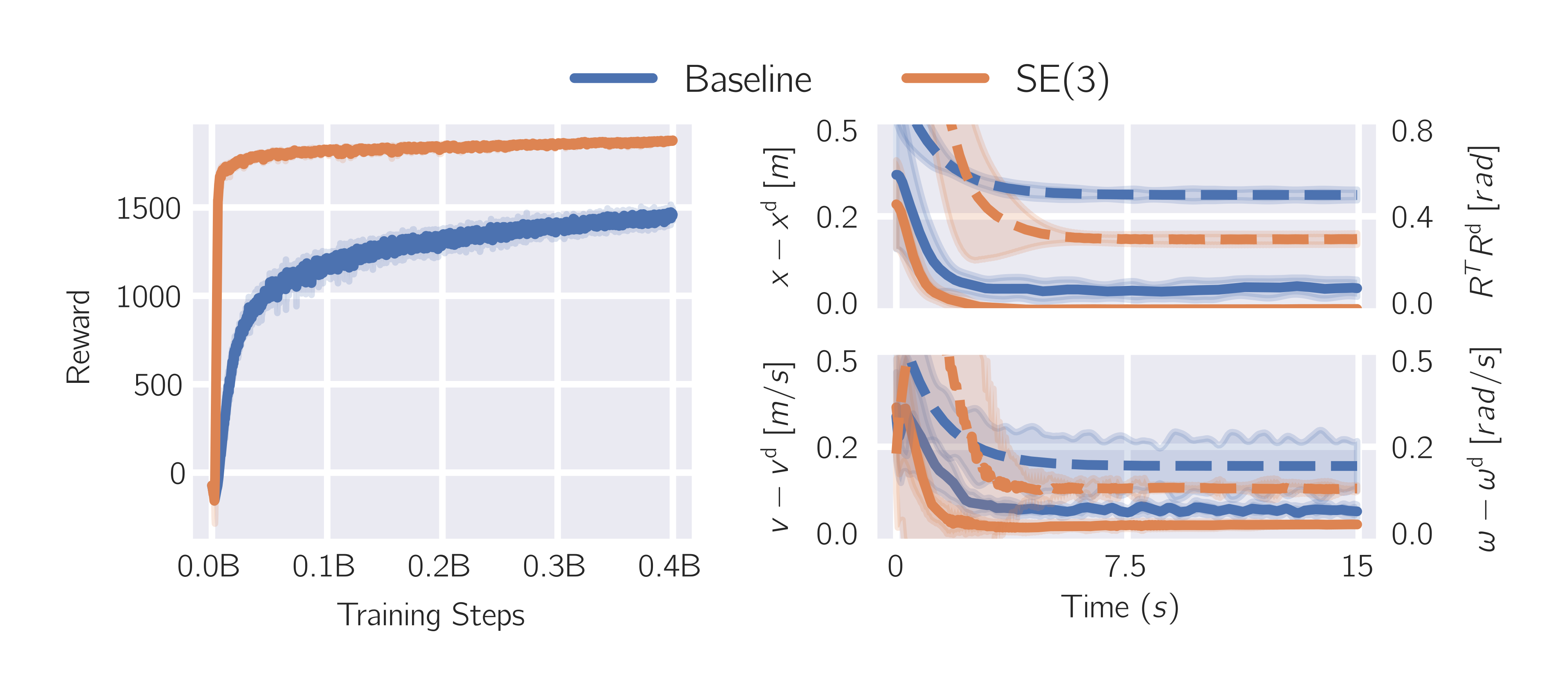}

\vspace{-10pt}
{\footnotesize (b)
Training and evaluation for
\texttt{Astrobee}.
}
\vspace{0pt}

\includegraphics[trim={0.25cm 0.10cm 0.25cm 0},clip,width=\columnwidth]{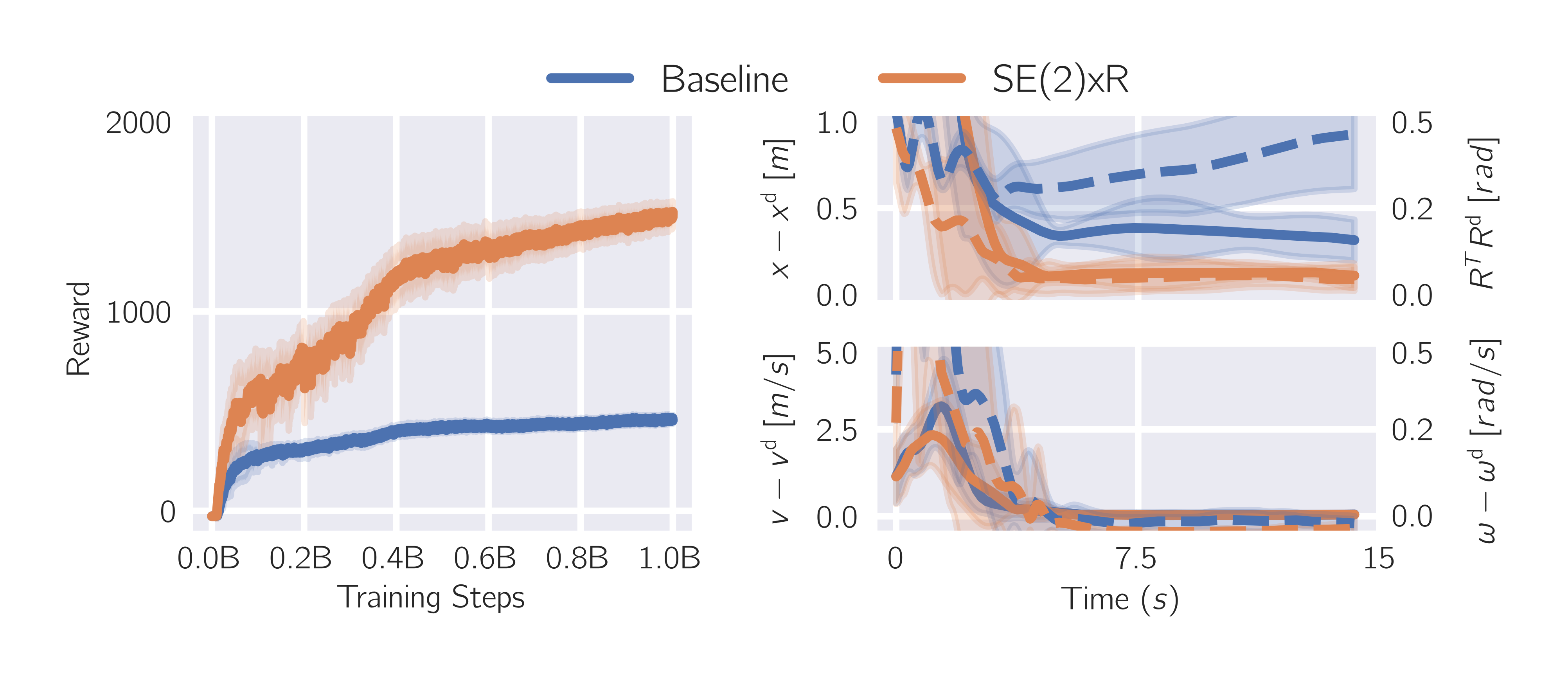}

\vspace{-10pt}
{\footnotesize (c) Training and evaluation for
 \texttt{Quadrotor}.
}
\caption{%
\hspace{-6pt}
Mean reward during training (for 10 training seeds) and, for the best-performing seed, mean tracking error during evaluation (for 20 trajectories),
with translational errors as solid lines and rotational errors as dashed lines.
}
\vspace{-14pt}
\label{fig:all_envs}
\end{figure}

\begin{table}[t]
\vspace{4pt}
\begin{center}
    \caption{\vspace{-1pt}Comparison of RMS Tracking Error on Planned Trajectories}
    \vspace{-2pt}
    \begingroup
    \setlength{\tabcolsep}{2.5pt} %
    \def\arraystretch{1.25}%
    \label{table:all_envs}

    \begin{tabular}{c  c  c c c c}
        Environment & $\mathcal{G}$ & $r \, \mathrm{[cm]}$ & $v \, \mathrm{[cm / s]}$ & $R \,\mathrm{[rad]}$  & $\omega\,\mathrm{[rad/s]}$ \\
        \hline
        \Xhline{1.5\arrayrulewidth}   
        \multirow{4}{*}{
\hspace{-3.5pt}\texttt{Particle}\hspace{-3.5pt}
        } & Baseline & 
298.3$\pm$1.2 & 103.0$\pm$1.1 & - & -
        \\
        & $\mathbb{R}^3$ & 
\hphantom{0}10.0$\pm$0.1 & \hphantom{00}4.8$\pm$0.1 & - & -
\\
        & $T\mathbb{R}^3$ & 
\hphantom{00}9.8$\pm$0.2 & \hphantom{00}4.3$\pm$0.0 & - & -
 \\
        & $T\mathbb{R}^3 \hspace{-1pt} \times \hspace{-1pt}\mathbb{R}^3$  & 
\pmb{\hphantom{00}8.3$\pm$0.1} & \pmb{\hphantom{00}3.9$\pm$0.2} & - & -
        \\ 
        \hline
        \multirow{2}{*}{
\hspace{-3.5pt}\texttt{Astrobee}\hspace{-3.5pt}
        } & Baseline & 
\hphantom{0}10.4$\pm$1.0 & \hphantom{00}6.0$\pm$1.0 & 0.75$\pm$0.01 & 0.24$\pm$0.02
        \\
        & 
        $SE(3)$
        & 
\pmb{\hphantom{00}1.3$\pm$2.7} & \pmb{\hphantom{00}3.3$\pm$3.1} & \pmb{0.37$\pm$0.04} & \pmb{0.16$\pm$0.06}
        \\
        \hline
        \multirow{2}{*}{
\hspace{-3.5pt}\texttt{Quadrotor}\hspace{-3.5pt}
        } & 
Baseline & 
\hphantom{0}66.3$\pm$3.2 & \hphantom{0}41.3$\pm$2.5 & 0.48$\pm$0.02 & 0.25$\pm$0.02
\\
& ${SE(2) \hspace{-1pt}\times\hspace{-1pt} \mathbb{R}}$ & 
\pmb{\hphantom{0}28.4$\pm$4.1} & \pmb{\hphantom{0}19.6$\pm$3.1} & \pmb{0.25$\pm$0.05} & \pmb{0.12$\pm$0.01}
 \\
        \hline
    \end{tabular}
    \endgroup
    
\end{center}    
{
        \vspace{-2pt}
We report the mean and standard deviation (for ${10}$ training seeds) of
the trained policy’s RMS tracking error (for a dataset of ${20}$ trajectories).}%
    \vspace{-10pt}
\end{table}

\section{Discussion}

\label{section:discussion}

Fig. \ref{fig:all_envs} shows a clear trend across the board: 
greater symmetry exploitation leads to improved sample efficiency. 
The tracking error evaluation shown in Table \ref{table:all_envs} and Fig. \ref{fig:all_envs} follows a similar trend. For the \texttt{Particle}, the vast majority of this benefit is achieved by reduction of the translational symmetry, although incorporating the velocity and force symmetries yields modest additional gains. This seems consistent with the large improvement we see for the \texttt{Astrobee} and \texttt{Quadrotor} 
after reduction by (a subgroup of) $SE(3)$.
Careful reward engineering 
or hyperparameter tuning
might improve performance
(especially for the baseline, which currently struggles to learn effectively),
but we 
instead
focus on analyzing the benefit of exploiting symmetry for a fixed reward. Nonetheless, any reward depending only on the reduced state $\tilde{s} = p(s)$ would preserve the symmetry.

Our approach assumes that at deployment, an upstream planner provides dynamically feasible reference trajectories. For the (underactuated) \texttt{Quadrotor}, these trajectories are planned using differential flatness \cite{Mellinger2011} from Lissajous curves in the flat space. However, in theory any other method (\textit{e.g.}, direct collocation \cite{Kelly2017}) could be used to generate a suitable reference.
We expect our policies to generalize well to a wide range of upstream planning methodologies, and future work should explore this hypothesis. Going forward, we also hope to apply these methods to new robot morphologies that are too complex for real-time numerical optimal control or for which no closed-form analytical controllers are known.

\section{Conclusion}

In this work, we exploit the natural Lie group symmetries of free-flying robotic systems to mitigate the challenges of learning trajectory tracking controllers. 
\label{section:conclusion}
We formulate the tracking problem as a single stationary MDP, proving that the underlying symmetries of the dynamics and running costs permit the reduction of this MDP to a lower-dimensional problem. When learning tracking controllers for space and aerial robots, training is accelerated and 
tracking error is reduced after the same number of training steps.
We believe our theoretical framework 
provides insight into the use of RL for systems with symmetry in robotics applications.

\balance 
\bibliographystyle{IEEEtran}
\bibliography{IEEEabrv,references}

\ifappendix

\begin{appendix}

The following Lemma provides a more detailed justification for a claim made in the proof of Theorem \ref{tracking_problem_symmetry}.
	\begin{lemma}
		Let ${\Upsilon: \manifold{G} \times \manifold{M} \to \manifold{M}}$ and 
		${\Theta : \manifold{H} \times \manifold{N} \to \manifold{N}}$ be free and proper group actions. Then,
		the {\normalfont ``product action''} 
		\begin{equation}
			\Pi_{(g,h)}(m,n) := \big(\Upsilon_g(m),\Theta_h(n)\big)
			\label{product_action}
		\end{equation}
		of $\manifold{G} \times \manifold{H}$ on $\manifold{M} \times \manifold{N}$ 
		and
		the {\normalfont ``diagonal action''}
		\begin{equation}
			\Gamma_g(m_1,m_2) := \big(\Upsilon_g(m_1),\Upsilon_g(m_2)\big)
			\label{diagonal_action}
		\end{equation}		
		of $\manifold{G}$ on ${\manifold{M} \times \manifold{M}}$
		are both free and proper.
	\end{lemma}
	\begin{proof}
		From \eqref{product_action}, ${\Pi_{(g,h)}(m,n) = (m,n)}$ implies that
		${\Upsilon_g(m) = m}$ and ${\Theta_h(n) = n}$ (and hence 
		${(g,h)=}$ ${(\identity_{\manifold{G}},\identity_{\manifold{H}})}$, since $\Upsilon$ and $\Theta$ are free). 
		Likewise, 
		from \eqref{diagonal_action}, ${\Gamma_g(m_1,m_2) = (m_1,m_2)}$ implies
		${\Upsilon_g(m_1) = m_1}$ (and hence
		${g = \identity_{\manifold{G}}}$, since $\Upsilon$ is free). 
		Thus, $\Gamma$ and $\Pi$ are free. 
		
		By definition, a group action ${\Phi : \manifold{G} \times \manifold{X} \to \manifold{X}}$ is proper if and only if ${\hat{\Phi} :  \manifold{G} \times \manifold{X} \to \manifold{X} \times \manifold{X}}$, ${(g,x) \mapsto \big(\Phi_g(x),x\big)}$ is a proper map (\textit{i.e.}, the preimage of any compact set is compact). 
		Since smooth manifolds are locally compact and Hausdorff, the product of continuous proper maps between them is also proper \cite[\S 1.5]{Dieck2008}. Observing that ${\hat{\Pi} : (g,h,m,n) \mapsto \big(\Upsilon_g(m),\Theta_h(n), m, n\big)}$ is (up to permutation of the components) the product map of $\hat{\Upsilon}$ and $\hat{\Theta}$, it follows that $\Pi$ is proper. 
		Additionally, since
		$\Upsilon$ is proper, the set ${\manifold{G}_C^\Upsilon = \{g \in \manifold{G} : C \cap \Upsilon_g(C) \neq \varnothing\}}$ is compact for every compact ${C \subseteq \manifold{M}}$ \cite[Prop. 21.5]{SmoothLee}.
		Considering any compact subset ${L \subseteq \manifold{M} \times\manifold{M} \times\manifold{M} \times\manifold{M} }$, we define ${K = \pr_1(L) \cup \pr_2(L) \cup \pr_3(L) \cup \pr_4(L)}$.
		Thus,
		\begin{align}
			&\hat{\Gamma}^{-1}(L) \subseteq
			\hat{\Gamma}^{-1}(K \times K \times K \times K)
			\\ &=
			\{(g,m_1,m_2) : \Upsilon_g(m_1), \Upsilon_g(m_2), m_1, m_2 \in K\}
			\\ &\subseteq \manifold{G}_K^\Upsilon \times K \times K.
		\end{align}
		The continuity of $\hat{\Gamma}$ implies that 
		$\hat{\Gamma}^{-1}(L)$ is a closed subset of the compact set ${\manifold{G}_K^\Upsilon \times K \times K}$ and is thus compact.
	\end{proof}

\end{appendix}

\nobalance

\fi

\clearpage

\end{document}